\documentclass[11pt]{article}
\usepackage[numbers]{natbib}
\usepackage{amsmath,amssymb,amsthm}
\usepackage{color}
\usepackage{geometry}
\usepackage{setspace}

\newcommand{\E}{\mathbb{E}}

\DeclareMathOperator*{\argmax}{argmax}

\theoremstyle{plain}

\newtheorem{thm}{Theorem}
\newtheorem{lemma}[thm]{Lemma}

\newtheorem{remark}{Remark}

\newtheoremstyle{TheoremNum}
        {\topsep}{\topsep}              
        {\itshape}                      
        {}                              
        {\bfseries}                     
        {.}                             
        { }                             
        {\thmname{#1}\thmnote{ \bfseries #3}}
\theoremstyle{TheoremNum}

\begin{document}
	\title{A Note on the Equivalence of Upper Confidence Bounds\\ and Gittins Indices for Patient Agents}
	\author{Daniel Russo}
	\maketitle
	\begin{abstract}
	This note gives a short, self-contained, proof of a sharp connection between Gittins indices and Bayesian upper confidence bound algorithms. I consider a Gaussian multi-armed bandit problem with discount factor $\gamma$.  The Gittins index of an arm is shown to equal the $\gamma$-quantile of the posterior distribution of the arm's mean plus an error term that vanishes as $\gamma\to 1$. In this sense, for sufficiently patient agents, a Gittins index measures the highest plausible mean-reward of an arm in a manner equivalent to an upper confidence bound. 
	\end{abstract}



\section{Introduction and Related Work}
There are two separate segments of the multi-armed bandit literature. One formulates a Bayesian multi-armed bandit problem as a Markov decision process and uses tools from dynamic programming to compute or approximate the optimal policy. This literature builds on a beautiful result that shows an optimal policy selects in each period the arm with highest Gittins index \citep{gittins1974dynamic, gittins1979dynamic}. 
A second segment of the literature focuses on simple heuristic algorithms-- which are often easy to adapt to settings in which exact dynamic programming is computationally intractable--and studies their performance through simulation and theoretical bounds on their regret  \cite{kaufmann2012bayesian,KL-UCB2013,rusmevichientong2010linearly,srinivas2012information}. This literature descends from a seminal paper by \citet{lai1985asymptotically} that shows the asymptotic growth rate of expected regret in a frequentist model is minimized by selecting in each period the arm with greatest upper-confidence bound. 
 
A sharp relationship between upper confidence bounds and the Gittins index of a patient agent (whose discount factor is close to 1) helps to unify these two segments of the literature. This provides enormous conceptual clarity, allowing the upper-confidence bounds of \citet{lai1985asymptotically} to be seen roughly as a generalization of and asymptotic approximation to the Gittins index. Unfortunately, such links seem to be known only to a few expert researchers. The goal of this short note is twofold. First, for Gaussian multi-armed bandit problems, it states an asymptotic equivalence between the Gittins index and a Bayesian upper confidence bound in a transparent form absent from the current literature. Second, the note gives short and elementary (if somewhat ugly) proofs that hopefully make this material accessible to large audience of researchers. 
 
Asymptotic links between Gittins indices and upper confidence bounds were first recognized by \citet{chang1987optimal}. That paper presents  a sophisticated asymptotic expansion of the solution of diffusion approximations to the optimal stopping problems defining a Gittins index. Unfortunately, the analysis is highly complex and is inaccessible to most multi-armed bandit researchers. Perhaps as a result, this pioneering work appears not to be widely known or cited\footnote{ According to Google scholar, \cite{chang1987optimal} was cited only once in 2018, while \cite{lai1985asymptotically} was cited well over 200 times.}. Hopefully, the transparent form of Theorem 1 along with its short proof will help to remedy this.


Like \cite{chang1987optimal}, most other closely related papers are focused on developing approximations to the Gittins index with the goal of simplifying computation \citep{yao2006some, chick2009economic, gutin2016optimistic}. This note was highly influenced by my reading of \cite{gutin2016optimistic}. The upper bound on the Gittins index developed in Section \ref{sec: upper} comes from analyzing their algorithm. Also related is work by \cite{burnetas2003asymptotic, kaufmann2012bayesian, lattimore2016regret}, who study the regret of Gittins index like policies for finite-horizon undiscounted multi-armed bandit problems. See also \citep{nino2011computing} for a derivation of a finite-horizon approximation of the Gittins index and related computational issues. 
This short note is distinguished from these related works in that (1) I study the  Gittins index as classically defined rather than the heuristic of \cite{nino2011computing} and (2) this note is designed to develop conceptual insight through a sharp link between Gittins indices and Bayesian upper confidence bounds, rather than to develop  accurate computational approximations or give a frequentist regret analysis.

\section{Formulation and Main Result}
While the Gittins index is eventually used in multi-armed bandit problems, it is calculated by considering a modified \emph{one-armed} bandit problem. Consider a single arm with uncertain quality $\theta$. When played at time $t$, the arm generates a reward $R_t$ with $R_t | \theta \sim N(\theta , \sigma^2_W)$. The  posterior distribution of $\theta$ given observed rewards $R_0,\ldots R_{t-1}$ is Gaussian. We write $\theta| R_0\ldots R_{t-1} \sim N(\mu_t, \sigma_t^2)$ where the posterior parameters evolve according to
\[ 
\mu_{t}=\frac{\sigma_{t-1}^{-2}\mu_{t-1} + \sigma^{-2}_{W} R_{t-1} }{\sigma_{t-1}^{-2} + \sigma^{-2}_{W} }
\] 
and 
\begin{equation}\label{eq: posterior var}
\sigma_{t}^2 = \left(  \frac{1}{\sigma_{t-1}^{2}} + \frac{1}{\sigma^{2}_{W}} \right)^{-1}= \left(\frac{1}{\sigma_0^{2}}+\frac{1}{t\sigma_W^2}\right)^{-1}.
\end{equation}
To define the Gittins index, we follow the interpretation of \citet{weber1992gittins}. Imagine that the right to play this arm is restricted and for each play the decision maker must pay a tax $\lambda$. Alternatively, in any period the agent may choose to retire and earn a reward of 0 thereafter. This can be cast as a Markov decision process where the state at time $t$ is $(\mu_t, \sigma_t^2)$, which serves as a sufficient statistic for the agent's posterior belief. The expected reward when playing the arm at state $(\mu_t, \sigma_t^2)$ is $\mu_{t}$. The agent's actions are simple: after the first period, given the current state of her beliefs, she can continue or can retire. The value function for this MDP can be written as
\begin{equation}\label{eq: value function def}
V_{\gamma}^{\lambda}(\mu, \sigma)=\sup_{\tau\geq 1} \E\left[\sum_{t=0}^{\tau } \gamma^{t} \left( \theta -\lambda \right)  \, \middle\vert \, 
\mu_0  = \mu, \sigma_0^2 = \sigma^2 \right]=\sup_{\tau\geq 1} \E\left[\sum_{t=0}^{\tau } \gamma^{t} \left( \mu_{t} -\lambda \right)  \, \middle\vert \, 
\mu_0  = \mu, \sigma_0^2 = \sigma^2 \right] 
\end{equation}
where the supremum is over stopping times $\tau\geq 1$ with respect to $(R_{0},R_1, R_2, \ldots)$. The equality is due to the tower property of conditional expectation. (See Appendix \ref{subsec: value function def}.)  
The Gittins index is the largest tax such that participating in this game is advantageous to the agent, written as 
\begin{equation}\label{eq: gittins def}
\lambda_{\gamma}\left(\mu, \sigma^2 \right) = \sup\left\{\lambda\in \mathbb{R}  \, \bigg\vert\, \,  V_{\gamma}^{\lambda}(\mu, \sigma) \geq 0 \right\}.
\end{equation}
This is interpreted sometimes as either a ``fair'' or ``prevailing'' tax.

To develop some intuition, note that for any tax $\lambda>\mu$ the agent could feasibly explore for some large number of periods and then choose to continue sampling only if posterior mean strictly exceeds the tax $\lambda$. For a very patient agent, the benefit of repeatedly playing an arm that generates rewards above the tax would dwarf the expected cost of initial exploration. As a result, the prevailing tax for the arm must be high enough that this event occurs very infrequently. The following theorem makes this intuition precise, showing that up to an error term that vanishes as $\gamma\to 1$, the Gittins index is eactly equal to the $\gamma$ quantile of the $N(\mu, \sigma^2)$ prior distribution of $\theta$. For sufficiently patient agents, a Gittins index measures the highest plausible mean-reward of an arm in a manner equivalent to a Bayesian upper confidence bound. 
\begin{thm} Fix any prior mean $\mu$ and prior variance $\sigma^2$.  Then, 
	\begin{equation}\label{eq: asymptotic approximation to gittins}
	\lambda_{\gamma}(\mu,\sigma^2) =\mu + \Phi^{-1}(\gamma)\sigma +o(1) \quad \text{as } \gamma \to 1,
	\end{equation}
	where $\Phi(\cdot)$ denotes the CDF of the standard normal distribution. 
\end{thm}
\begin{remark}
	While the Gittins index is derived by considering a one-armed bandit problem, both Gittins indices and upper confidence bounds are usually applied in bandit problems with $k>1$ arms. In that context, the Gittins index theorem shows an optimal policy plays at time $t$ the arm  $\argmax_{i\leq k} \lambda_{\gamma}(\mu_{t,i}, \sigma^2_{t,i})$ whose posterior parameters $(\mu_{t,i}, \sigma_{t,i})$ are associated with the maximal Gittins index. A Bayesian UCB algorithm plays the arm $\arg\max_{i\leq k} \mu_{t,i} + \sigma_{t,i} \Phi^{-1}(q_{t})$, where the posterior quantile $q_t$ is often treated as a tunable parameter and theory suggests values like $q_{t}= 1-1/T$ when there is a known time-horizon of $T$. The quantile in \eqref{eq: asymptotic approximation to gittins} is then analogous to using the natural time horizon of $T=1/(1-\beta)$ for a discounted problem. 
\end{remark}
\begin{remark}
	For readers more familiar with the upper-confidence bounds of \citet{auer2002finite} than the Bayesian form presented here, it is worth noting that these expressions are almost identical if an improper prior is used or an arm has been sampled a moderate number of times. More about these connections can be found in \cite{kaufmann2012bayesian}. 
\end{remark}

\section{Analysis}

\subsection{Technical Preliminaries}
\paragraph{Strict concavity of the square root.} The next lemma is used several times in the analysis.  The idea is that because $g(x)=\sqrt{x}$ is strictly concave and $g'(x)\to 0$ as $x\to \infty$ , $\sqrt{x+y} \approx \sqrt{x}$ if $x$ is much larger than $y$. 
\begin{lemma}
	\label{lem: sqrt concavity}
	Let $f: \mathbb{R}_{+} \to \mathbb{R}_{+}$ be any function satisfying $|f(x)| = o(\sqrt{x})$ as $x\to \infty$. Then, 
	\[ 
	\sqrt{x+f(x)} = \sqrt{x} + o(1) \quad \text{as } x\to \infty.   
	\]
\end{lemma}
\begin{proof} 
	By Taylor's theorem, there is some $\tilde{x} \in [x, x+f(x)]$ such that
	\[
	\sqrt{x+f(x)} - \sqrt{x}  = \frac{f(x)}{2\sqrt{x}}+ \frac{1}{2}\left( \frac{-f(x)}{\tilde{x}^{3/2}}\right) 
	\]
	For $f(x)=o(\sqrt{x})$, both terms on the right hand side vanish as $x\to \infty$. 
\end{proof}

\paragraph{Gaussian tail behavior.}
Let  $\phi(z) = \frac{1}{\sqrt{2\pi}} e^{-z^2 /2 }$ denote the PDF of the standard normal distribution. Because the PDF decays exponentially as $z$ increases, for large values of $z$, tail integrals like $1-\Phi(z)=\intop_{z'>z} \phi(z') dz'$ also decay exponentially as $e^{-z^2/2}$ as $z\to \infty$. Inverting this relation suggests an asymptotic approximation of $\Phi^{-1}(\gamma) \approx \sqrt{2 \log(1/(1-\gamma))}$ to the quantiles of the normal distribution. The next lemma, proved in \ref{subsec: proof of normal quantile}, makes this precise. 
\begin{lemma}\label{lem: normal quantile}
	As $\gamma \to \infty$,
	\[
	\Phi^{-1}(\gamma) =  \sqrt{2 \log\left(\frac{1}{1-\gamma} \right)} + o(1).
	\]
\end{lemma}
The same type of saddle-point approximation shows the integral $\E[(Z-z)^+] = \intop_{z'>z} z'\phi(z') dz'$ decays like $e^{-z^2/2}$ as $z\to \infty$. For our analysis, it is convenient to have explicit upper and lower bounds, like those in the following lemma. The upper bound here is a standard Gaussian maximal inequality and the lower bound applies Lemma 3 in \cite{qin2017improving}.  
\begin{lemma}
	\label{lem: truncated normal}
	For $X\sim N(0, \sigma^2)$ and $\lambda \geq \mu + 2\sigma$,
	\[ 
	\frac{\sigma^4}{\lambda^3} \phi\left(  \frac{\lambda}{\sigma} \right) \leq \E[(X-\lambda)^{+}]\leq \sigma \phi\left(  \frac{\lambda}{\sigma} \right).
	\]
\end{lemma}




\subsection{Reduction to Indices for Standard Normal Distributions}
With some abuse of notation, for the moment let us explicitly capture the dependence of the Gittins index on the noise variance, setting
 $\lambda_{\gamma}(\mu, \sigma^2, \tilde{\sigma}^2) $ to be the Gittins index for a bandit process with prior mean $\mu$, prior variance $\sigma^2$ and noise variance $\tilde{\sigma}^2 $. A simple standardization argument shows \cite{gittins2011multi}
\[
\lambda_{\gamma}(\mu, \sigma^2, \tilde{\sigma}^2) = \mu + \sigma \lambda_{\gamma}\left(0, 1, \frac{\tilde{\sigma}^2}{\sigma^2} \right).
\]
Therefore it suffices to study the Gittins index for an arm with standard normal prior and some arbitrary noise variance we denote by $\sigma_{W}^2$. Combining this with Lemma \ref{lem: normal quantile}, our goal in subsequent subsections is to show $\lambda_{\gamma}(0,1) = \sqrt{2\log(1/(1-\gamma))}+o(1)$  as $\gamma \to 1$, where we treat $\sigma_{W}^2>0$ as an arbitrary positive constant interpreted as the noise-to-signal ratio.

\subsection{Upper Bound on the Gittins index}\label{sec: upper}
This subsection derives an upper bound on the Gittins index via an information relaxation \cite{brown2010information}. We consider a decision-maker who reveals noiseless signals of the true arm mean $\theta$ when she samples the arm. The prevailing tax for this decision-maker exceeds the prevailing tax for one who must base their decisions on noisy reward signals. As $\gamma\to 1$, this upper bound  matches both a lower bound given in Lemma \ref{lem: Gittins lower bound} and the posterior quantile in Lemma \ref{lem: normal quantile}. 
\begin{lemma}\label{lem: Gittins upper bound}
	\[
	\lambda_{\gamma}\left(0,1 \right) \leq \sqrt{2\log\left(\frac{1}{1-\gamma} \right)} + o(1) \quad \text{as  } \gamma \to 1.
	\]
\end{lemma}
\begin{proof}
To simplify notation, write $\lambda_{\gamma} = \lambda_{\gamma}\left(0,1 \right)$ and note that we often use $\E[\theta]=0$ to simplify expressions. Consider a decision-maker who faces a one-armed bandit problem with no observation noise. For this decision-maker, playing the arm once is sufficient to perfectly reveal the true arm mean $\theta$. An optimal policy would then play the arm in every period if $\theta\geq \lambda$, and immediately retire otherwise. Of course, a Bayesian decision-maker is better off basing her retirement decision on perfect knowledge of $\theta$ than on noisy signals (See e.g. \cite{degroot1962uncertainty}). This can be verified directly in this case: the decision-maker with access to noiseless observations earns
\begin{equation}\label{eq: upper bound on value function}
-\lambda + \left(\frac{\gamma}{1-\gamma}\right)\E\left[\left(\theta - \lambda \right)^{+} \right] = \E[\theta - \lambda] + \E \sum_{t=1}^{\infty} \left(\theta - \lambda \right)^{+}  \geq \sup_{\tau>0} \E \sum_{t=0}^{\infty} \left(\theta - \lambda \right) \mathbf{1}(\tau \geq t) = V^{\lambda}_{\gamma}(0,1). 
\end{equation}
Therefore, the fair tax for the decision-maker who observes noiseless signals of $\theta$ exceeds the fair tax $\lambda_{\gamma}$ for one who must base her stopping decision on imperfect signals. (See also \cite{gutin2016optimistic} for a detailed proof.). We have \begin{align*}
	\lambda_{\gamma} := \sup\left\{ \lambda \in \mathbb{R}  \mid    V_{\gamma}^{\lambda}(0,1) \geq \lambda \right\} &\leq  \sup\left\{ \lambda \in \mathbb{R}  \mid    \frac{\gamma}{1-\gamma}\E\left[ \left( \theta - \lambda  \right)^{+} \right] \geq \lambda \right\}\\ &\overset{\rm Lem. \ref{lem: truncated normal}}{\leq} \sup\left\{ \lambda \in \mathbb{R}  \mid    \frac{\gamma}{1-\gamma}\phi(-\lambda) \geq \lambda \right\}  \\
	&= \sup\bigg\{ \lambda \in \mathbb{R}  \mid  \log\left(\frac{\gamma}{1-\gamma}\right) \geq \log\left(\frac{\lambda}{\phi(-\lambda)} \right)   \bigg\} := \overline{\lambda}_{\gamma}.
\end{align*}
Plugging in for the normal PDF $\phi$ and simplifying, we find the upper bound $\overline{\lambda}_{\gamma}$ on the Gittins index is defined implicitly by
\begin{equation}\label{eq: defining upper bound}
\sqrt{2\log(\overline{\lambda}_{\gamma}) + \overline{\lambda}_{\gamma}^2} = \sqrt{2\log\left(\frac{1}{1-\gamma}\right)  + 2\log\left( \gamma \sqrt{2\pi}\right)}.
\end{equation}
As  $\gamma\to 1$, the right hand side tends to infinity and by Lemma \ref{lem: sqrt concavity}
\begin{equation}\label{eq: simplifying upper bound 1}
\sqrt{2\log\left(\frac{1}{1-\gamma}\right)  + 2\log\left( \gamma \sqrt{2\pi}\right)} = \sqrt{2\log\left(\frac{1}{1-\gamma}\right)}+ o(1).
\end{equation}
This implies that  $\overline{\lambda}_{\gamma} \to \infty$ as $\gamma \to 1$. Applying Lemma \ref{lem: sqrt concavity} again shows
\begin{equation}\label{eq: simplifying upper bound 2}
\sqrt{2\log(\overline{\lambda}_{\gamma}) + \overline{\lambda}_{\gamma}^2}  = \overline{\lambda}_{\gamma} + o(1)  \quad \text{as  } \gamma \to 1.
\end{equation}
Combining Equation \eqref{eq: defining upper bound} with \eqref{eq: simplifying upper bound 1} and \eqref{eq: simplifying upper bound 2} establishes the claim. 
\end{proof}

\subsection{Lower Bound on the Gittins index}
We construct a lower bound on the Gittins index by analyzing the fair tax for an agent who employs a suboptimal heuristic policy. This agent explores for a predetermined number of periods $L$. Based on the resulting signals, she retires if $\mu_{L}<\lambda$ and otherwise commits to playing the arm indefinitely. The main idea is that large $L$ will almost perfectly reveal $\theta$, but as $\gamma \to 1$ the cost of this initial exploration is small relative to the potential value from discovering the arm has very high quality and hence has a negligible impact on the fair tax for the game. The proof will choose $L$ as a slowly growing function of $\gamma$, so that the lower bound constructed here matches the upper bound in Lemma \ref{lem: Gittins upper bound} as $\gamma \to 1$. Specifically, a choice of $L_{\gamma}= \lceil\sigma_{W}^2\log( 1/ (1-\gamma) )^2\rceil$ suffices for the proof.
Note that this result matches the posterior quantile in Lemma \ref{lem: normal quantile}, and, together with Lemma \ref{lem: Gittins upper bound}, completes the proof of Theorem 1. 
\begin{lemma}\label{lem: Gittins lower bound}
	\[
	\lambda_{\gamma}\left(0,1 \right) \geq \sqrt{2\log\left(\frac{1}{1-\gamma} \right)} + o(1) \quad \text{as  } \gamma \to 1.
	\]
\end{lemma}
\begin{proof}
Consider a decision-maker who faces a tax $\lambda$. Suppose the agent follows a policy of exploring for $L\in \mathbb{N}$ periods, and then either retiring if $\mu_{L}< \lambda$ or playing the arm for all future periods otherwise.  The value of this heuristic policy is a lower bound on the optimal policy, so for all fixed $L\in \mathbb{N}$ 
\[
V_{\gamma}^{\lambda}(0,1) \geq  -\sum_{t=0}^{L-1} \gamma^{t} \lambda +  \frac{\gamma^{L}}{1-\gamma}\E\left[ \left( \mu_{L} - \lambda  \right)^{+} \right] 
\geq  -L\lambda +  \frac{\gamma^{L}}{1-\gamma}\E\left[ \left( \mu_{L} - \lambda  \right)^{+} \right]. 
\]
Define $\mathcal{H}_{L-1}=(R_0, \ldots R_{L-1})$ to be the history of rewards prior to period $L$. The posterior mean is random due to its dependence on $\mathcal{H}_{L-1}$ and has distribution $\mu_L \sim N(\mu_0, 1- \sigma_L^{2})$. Here normality follows from the fact that $\mu_L$ is a linear combination of Gaussian observations $R_1,\ldots R_{L-1}$, we have $\E[\mu_L]=\E[\E[\theta |\mathcal{H}_{L-1}]] = \mu_0$ by the tower property of conditional expectation, and the variance formula follows from the law of total variance:
\begin{align*}
1={\rm var}(\theta)= {\rm var}\left(\E[ \theta| \mathcal{H}_{L-1}] \right) + \E\left[{\rm var}\left( \theta | \mathcal{H}_{L-1} \right) \right]
 = {\rm var}(\mu_L) + \sigma_{L}^2.
\end{align*}
%
This implies that for any $L\in \mathbb{N}$, 	
\begin{align*} 
\lambda_{\gamma} & \geq \sup\left\{ \lambda \in \mathbb{R} \mid    \frac{\gamma^{L}}{1-\gamma}\E\left[ \left( \mu_{L} - \lambda  \right)^{+}\right] \geq L\lambda \right\} \\  
& \overset{\rm Lem. \ref{lem: truncated normal}}{\geq}  \sup\left\{ \lambda \in \mathbb{R} \mid    \frac{\gamma^{L}}{1-\gamma} \frac{\left(1-\sigma_L^2\right)^2}{\lambda^3} \phi\left(  \frac{\lambda}{\sqrt{1-\sigma_{L}^2} } \right) \geq L\lambda \right\}  \\  
&=   \sup\left\{ \lambda \in \mathbb{R} \bigg\vert  \log \left(\frac{\gamma^{L}}{1-\gamma}\right) + \log \left( \frac{\left(1-\sigma_L^2\right)^2}{\lambda^3}\right) 
\geq \log\left(L\lambda\right)- \log \phi\left(  \frac{\lambda}{\sqrt{1-\sigma_{L}^2} } \right) \right\}.
\end{align*}
Now, choose $ L_{\gamma} = \lceil \sigma^2 \log\left( \frac{1}{1-\gamma}  \right)^2 \rceil $, which tends slowly to infinity as $\gamma\to 1$, and set $\underline{\lambda}_{\gamma}$ to be the lower bound corresponding to the choice of $L=L_{\gamma}$. Plugging in for the normal PDF and simplifying, we find $\underline{\lambda}_{\gamma}$ is defined implicitly by 
\begin{equation}\label{eq: gittins lower bound}
\sqrt{4 \log(\underline{\lambda}_{\gamma}) + \frac{\underline{\lambda}_{\gamma}^2}{2 (1-\sigma_{L_\gamma}^2)} }  = \sqrt{ \log\left( \frac{1}{1-\gamma} \right)+h(\gamma) }
\end{equation}
where $h(\gamma):=-\log\left( L_{\gamma} \right) +L_{\gamma}\log(\gamma) + 2\log(1- \sigma_{L_\gamma}^2)+
\log(\sqrt{2\pi}).$  We want to focus on the dominant terms on each side of equation \eqref{eq: gittins lower bound}, which are  $\underline{\lambda}_{\gamma}^2
/2 (1-\sigma_{L_{\gamma}}^2)$ and $\log\left( \frac{1}{1-\gamma}  \right)$. The next result shows the $h(\gamma)$ term has an asymptotically negligible influence.
\begin{lemma}\label{lem: bounding h}
	$h(\gamma) = o(\sqrt{\log(1/\gamma)})$ as $\gamma\to 1$. 
\end{lemma}
Together with Lemma \ref{lem: sqrt concavity}, this shows $\sqrt{\log(1/\gamma)+h(\gamma)}= \sqrt{\log(1/\gamma)}+o(1)$ as $\gamma \to 1$. Hence, the solution $\underline{\lambda}_{\gamma}$ to 
equation \eqref{eq: gittins lower bound} must also tend to $\infty$ as $\gamma \to 1$. Then, again by Lemma \ref{lem: sqrt concavity},
\begin{equation}\label{eq: simpler gittins lower bound LHS}
\sqrt{4 \log(\underline{\lambda}_{\gamma}) + \frac{\underline{\lambda}_{\gamma}^2}{2 (1-\sigma_{L_\gamma}^2)} } = \frac{\underline{\lambda}_{\gamma}}{\sqrt{2 (1-\sigma_{L_\gamma}^2)}}+o(1).
\end{equation}
Combining Equations \eqref{eq: gittins lower bound} and \eqref{eq: simpler gittins lower bound LHS} gives
\begin{equation}\label{eq: simplified Gittins lower bound}
\underline{\lambda}_{\gamma} = \sqrt{2(1-\sigma_{L_\gamma}^2)\log\left( \frac{1}{1-\gamma} \right)}+o(1). 
\end{equation}
The only remaining subtlety is the term $(1-\sigma_{L_\gamma}^2)$, which appears here since after $L_{\gamma}$ measurements the agent still has some remaining uncertainty about the value of $\theta$. From the formula \eqref{eq: posterior var} for posterior variance,  $\sigma_{L_\gamma}^2 \leq \sigma_W^2/L_{\gamma}$. Plugging in for $L_{\gamma}= \lceil\sigma_W^2\log( 1/ (1-\gamma) )^2\rceil $ gives $\sigma^2_{L_{\gamma}} \log( 1/ (1-\gamma) ) \leq 1$. Plugging this into \eqref{eq: simplified Gittins lower bound} gives
\begin{align*}
\underline{\lambda}_{\gamma} \geq   \sqrt{2\log\left( \frac{1}{1-\gamma}\right) -2 } + o(1) 
=  \sqrt{2\log\left( \frac{1}{1-\gamma}\right)}+o(1).
\end{align*}
\end{proof}

\section{Limitations and Open Problems}
While this note shows an equivalence between a Gittins index and a Bayesian upper-confidence bound, it should be stressed that this equivalence is asymptotic as the effective time-horizon of the problem grows. In particular, the Gittins index carefully captures the value of exploration given the time horizon of the problem and the variance of reward noise. Upper confidence bound algorithms do not and can engage in wasteful exploration if there is significant observation noise relative to the problem's time horizon.

One natural open direction is to extend Theorem 1 and its proof to single parameter exponential family distributions. Another question is whether extensions of the analysis can yield appropriate uniform or functional limit theorems analogous to Theorem 1. This is important to providing frequentist regret analysis of Gittins index algorithms or Bayesian regret analysis of UCB approximations. See \cite{chang1987optimal,lattimore2016regret}. 

\section*{Acknowledgements}
Much of this short note was written as material for a doctoral course taught at Northwestern in Spring 2017. I am grateful to the students in that course for their questions and feedback. 

\begin{singlespace}
	
\bibliography{references}
\bibliographystyle{plainnat}
\end{singlespace}

\appendix

\section{Omitted Technical Details}

\subsection{Proof of Lemma \ref{lem: normal quantile}}\label{subsec: proof of normal quantile}
\begin{proof}
	We use the following standard bounds on the Normal CDF \cite{gordon1941values}: for all $z\geq 0$,  
	\[ 
	\left(\frac{z}{1+z^2}\right) \phi(z) \leq 1-\Phi(z) \leq \left(\frac{1}{z}\right) \phi(z)
	\]
	We can use this to  upper bound $\Phi^{-1}(\gamma)$ as follows: 
\begin{align*}
	\Phi^{-1}(\gamma) = \inf\{z \in \mathbb{R} \mid \Phi(z) \geq 1- \gamma \} 
	&\leq \inf\{z \in \mathbb{R} \mid   \left(\frac{1}{z}\right) \phi(z) \geq 1- \gamma \} \\
	&= \inf\{z \in \mathbb{R} \mid   \log\left(\frac{\phi(z)}{z}\right)  \geq \log (1- \gamma) \}:= \overline{z}_\gamma.
\end{align*}
Plugging in for the normal PDF $\phi$ and simplifying, we find that $\overline{z}_{\gamma}$ is defined implicitly by 
\begin{equation}\label{eq: upper bound on normal quantile} 
\sqrt{\overline{z}_{\gamma}^2 + 2\log\left(\overline{z}_{\gamma}\sqrt{2\pi}\right)}  = \sqrt{2\log\left(\frac{1}{1-\gamma} \right)}. 
\end{equation}
As $\gamma \to 1$, the right hand side of \eqref{eq: upper bound on normal quantile} tends to $\infty$, so it must be that $\overline{z}_{\gamma} \to \infty$. But since  $\log\left(\overline{z}_{\gamma}\sqrt{2\pi}\right) = o(\sqrt{\overline{z}_{\gamma}})$ as $\gamma \to 1$, applying Lemma \ref{lem: sqrt concavity} gives $\sqrt{\overline{z}_{\gamma}^2 + 2\log\left(\overline{z}_{\gamma}\sqrt{2\pi}\right)} = \overline{z}_\gamma + o(1)$. We conclude 
\[ 
\overline{z}_{\gamma} =  \sqrt{2\log\left(\frac{1}{1-\gamma} \right)} + o(1) \quad \text{as  } \gamma \to 1.  
\]
The proof of the lower bound follows the same steps and is omitted. 
\end{proof}

\subsection{Proof of Lemma \ref{lem: bounding h}}
 We show $h(\gamma)=o(\sqrt{\log(1/\gamma)})$ as $\gamma \to 1$. We evaluate each term in the expression $h(\gamma):=-\log\left( L_{\gamma} \right) +L_{\gamma}\log(\gamma) + 2\log(1- \sigma_{L_\gamma}^2)+
  \log(\sqrt{2\pi})$. Since $\log(\gamma) = -(1-\gamma) + o(1-\gamma)$ as $\gamma\to 1$, we have $L_{\gamma} \log(\gamma) \to 0$. In addition, $2\log(1- \sigma_{L_{\gamma}}^2)\to 0$ since by \eqref{eq: posterior var}, $\sigma_{L_{\gamma}}^2 \leq \sigma_W^2 / L_\gamma  \to 0$ as $\gamma \to 1$.  Finally, $\log(L_{\gamma}) =2\log(\sigma)+ 2\log \log\left(\frac{1}{1-\gamma}\right) = o\left( \sqrt{\log\left(\frac{1}{1-\gamma}\right)} \right)$.
  
\subsection{Further Justification for Equation \ref{eq: value function def}.}\label{subsec: value function def} 
Equation \eqref{eq: value function def} relies on Doob's optional-sampling theorem. Here we  note the technical conditions ensuring this applies. Let $\mathcal{H}_{t}$ denote the sigma-algebra generated by $R_0,\ldots, R_{t-1}$ and let $\tau$ be any stopping time with respect to $\{\mathcal{H}_{t} : t\in 0,1,\ldots \}$. Define the martingale  $M=\{M_{n} : n=0,1,\ldots\}$  by 
\[
M_n = \sum_{t=0}^{n} \gamma^{t} (\theta - \E[\theta \mid \mathcal{H}_{t-1}]). 
\]
For each fixed $n$, $\E[M_n]=0$. Equation \eqref{eq: value function def} states that $\E[M_{\tau}]=0.$ (To compare, recall the definition $\mu_t=\E[\theta \mid \mathcal{H}_{t-1}]$). This follows by Doob's optional sampling theorem since $M$ is a uniformly integrable martingale. To show $M$ is uniformly integrable, it suffices to show it is bounded in $L^2$. We have 
\begin{align*}
\sup_{n} \E\left[ M_n^2 \right] = \sup_{n} \sum_{t=0}^{n} \gamma^t \mathbb{E} \left[ (\theta - \E[\theta\mid \mathcal{H}_{t-1}]  )^2\right] = \sum_{t=0}^{\infty} \gamma^{t}\E\left[{\rm Var}(\theta \mid \mathcal{H}_{t-1})\right]  \leq  \sum_{t=0}^{\infty} \gamma^{t} {\rm Var}(\theta) < \infty  \end{align*}
where the inequality $\E\left[{\rm Var}(\theta \mid \mathcal{H}_{t-1})\right]  \leq {\rm Var}(\theta) $ is standard and follows from Jensen's inequality for conditional expectations. 
\end{document}